\documentclass[twoside,leqno,twocolumn]{article}

\usepackage[letterpaper]{geometry}

\usepackage{ltexpprt}
\usepackage{hyperref}
\usepackage{amsmath}
\usepackage{amsfonts}
\usepackage{algorithm}
\usepackage{algpseudocode}
\usepackage{float}
\usepackage{booktabs}
\usepackage{graphicx}
\usepackage{siunitx}

\begin{document}

%
%\newcommand\relatedversion{}
%\renewcommand\relatedversion{\thanks{The full version of the paper can be accessed at \protect\url{https://arxiv.org/abs/1902.09310}}} % Replace URL with link to full paper or comment out this line

%\setcounter{chapter}{2} % If you are doing your chapter as chapter one,
%\setcounter{section}{3} % comment these two lines out.

%\title{\Large SIAM/ACM Preprint Series Macros for Use With LaTeX\relatedversion}
\title{\Large A Simple Sparse Matrix Vector Multiplication Approach to Padded Convolution}
\author{Zan Chaudhry\thanks{Radiology AI Laboratory, Johns Hopkins School of Medicine, Baltimore, MD, 21205 USA.}}
%\and Tricia Manning\thanks{Society for Industrial and Applied Mathematics.}}

\date{}

\maketitle

% Copyright Statement
% When submitting your final paper to a SIAM proceedings, it is requested that you include
% the appropriate copyright in the footer of the paper.  The copyright added should be
% consistent with the copyright selected on the copyright form submitted with the paper.
% Please note that "20XX" should be changed to the year of the meeting.

% Default Copyright Statement
%\fancyfoot[R]{\scriptsize{Copyright \textcopyright\ 20 by SIAM\\
%Unauthorized reproduction of this article is prohibited}}

% Depending on which copyright you agree to when you sign the copyright form, the copyright
% can be changed to one of the following after commenting out the default copyright statement
% above.

%\fancyfoot[R]{\scriptsize{Copyright \textcopyright\ 20XX\\
%Copyright for this paper is retained by authors}}

%\fancyfoot[R]{\scriptsize{Copyright \textcopyright\ 20XX\\
%Copyright retained by principal author's organization}}

%\pagenumbering{arabic}
%\setcounter{page}{1}%Leave this line commented out.

\begin{abstract} \small\baselineskip=9pt 
We introduce an algorithm for efficiently representing convolution with zero-padding and stride as a sparse transformation matrix, applied to a vectorized input through sparse matrix-vector multiplication (SpMV). We provide a theoretical contribution with an explicit expression for the number of non-zero multiplications in convolutions with stride and padding, offering insight into the potential for leveraging sparsity in convolution operations. A proof-of-concept implementation is presented in \texttt{Python}, demonstrating the performance of our method on both CPU and GPU architectures. This work contributes to the broader exploration of sparse matrix techniques in convolutional algorithms, with a particular focus on leveraging matrix multiplications for parallelization. Our findings lay the groundwork for future advancements in exploiting sparsity to improve the efficiency of convolution operations in fields such as machine learning and signal processing.
\end{abstract}

\section{Introduction}
Signal processing and machine learning methods are both experiencing tremendous growth in the age of big data and intelligent systems. And both are dependent on two simple linear transformations: convolution and padding. Convolution ($*$) is defined as:

\begin{equation}
f (t) * g(t) = \int_{-\infty}^{\infty} f(\tau) g(t-\tau) d\tau
\end{equation}

In the finite discrete case with two-dimensional input (as in typical applications) with $f, g \in \mathbb{R}^{M\times N}$ this reduces to:

\begin{equation}
f [j,k] * g[j,k] = \sum_{m=1}^{M}\sum_{n=1}^{N} f[m, n] g[j-m, k-n]
\end{equation}

Convolution can also be implemented with stride ($s$), in which $m,n$ are incremented by integers, $s_1, s_2>1$ (we consider $s_1=s_2$, though the results can easily be extended to other cases). There are many forms of padding, but here we consider zero-padding, which is by far the most common. In zero-padding, an input, $\mathbf{A} \in \mathbb{R}^{m \times n}$, is reshaped to $\mathbf{A}_{\text{pad}} \in \mathbb{R}^{m+2p_1 \times n+2p_2}$ by surrounding the input with $p_1$ rows of zeros above and below the input and $p_2$ columns of zeros on the left and right (we consider $p_1=p_2$, though the results can easily be extended to other cases).

Convolution and padding arise in many current scientific computing problems, particularly image processing contexts and convolutional neural networks (CNNs). Sophisticated libraries and algorithms have been developed to implement these operations and leverage engineering advancements in computer architectures \cite{harris2020array, 2020SciPy-NMeth, 10.5555/3454287.3455008, tensorflow2015-whitepaper}. Particularly, matrix multiplication has been highly optimized through parallelized processing, and numerous methods have emerged for casting discrete convolution as matrix multiplication \cite{blackford2002updated, mkl_dnn, cudnn, anderson2017lowmemorygemmbasedconvolutionalgorithms}. 

Typical convolution algorithms that cast convolution as matrix multiplication have time complexity $O(m_\text{out}\times n_{\text{out}} \times k^2 )$, where $m_{\text{out}}, n_{\text{out}}$ are the dimensions of the output of the operation, and a $k \times k$ kernel is applied to the $m \times n$ input (we consider square kernels, though the results can easily be extended to other cases) \cite{chellapilla:inria-00112631}. Padding is implicit in this construction. We propose a sparse matrix vector multiplication (SpMV) algorithm for convolution with padding of reduced time-complexity:
\begin{equation*}
O\Bigg(\sum_{x=0}^{m_{\text{out}}-1} \text{ } \sum_{y=0}^{n_{\text{out}}-1} \text{max}\Big(0, k-c_{1}(x)\Big) 
\text{ max}\Big(0, k-c_{2}(y)\Big) \Bigg)
\end{equation*}
where:
\begin{align*}
c_{1}(x) &= \text{max}(0, p-sx) + \text{max}(0, sx+k-m-p) \\
c_{2}(y) &= \text{max}(0, p-sy) + \text{max}(0, sy+k-n-p)
\end{align*}

Our algorithm represents both convolution and padding as matrices that act on a vectorized input matrix. This method requires a one-time cost of constructing the transformation matrices, after which convolutions are simply SpMV operations. In contrast, popular current methods (particularly the \texttt{im2col} method) transform the input for each convolution and do not leverage sparsity \cite{chellapilla:inria-00112631, zhou2021characterizingdemystifyingimplicitconvolution}. We present the mathematical motivation for our method and provide proof-of-concept implementations in \texttt{Python} \cite{10.5555/1593511}. We provide CPU implementations using \texttt{SciPy} and \texttt{NumPy} and GPU implementations using \texttt{PyTorch} and \texttt{CuPy} \cite{2020SciPy-NMeth, harris2020array, 10.5555/3454287.3455008, cupy_learningsys2017}. We compare CPU and GPU implementations to the \texttt{Conv2D} method in \texttt{PyTorch}. 

\section{Derivation}

Consider an $m \times n$ input, which we pad to $m +2p \times n+2p$ and then convolve with  $\mathbf{K}$, a $k \times k$ kernel with stride $s$. We begin by vectorizing the input:

\begin{align}
\mathbf{A} &= \begin{bmatrix}
A_{11} & \cdots & A_{1n} \\
\vdots & \ddots & \vdots \\
A_{m1} &  \cdots & A_{mn} \\ 
\end{bmatrix} \nonumber \\
\Vec{a} = \text{vec}(\mathbf{A}) &= \begin{bmatrix} A_{11} & A_{12} & A_{13} & ... & A_{mn} \end{bmatrix}^{\text{T}}
\end{align}

We define the padding matrix, $\mathbf{P}$, such that the matrix form of the vector $\Vec{a}_{\text{pad}} = \mathbf{P}\Vec{a}$ is equivalent to $\mathbf{A}$ surrounded by zeros. We can define such a matrix, $\mathbf{P} \in \mathbb{R}^{(m+2p)(n+2p) \times (mn)^2}$, as:

\begin{align*}
\renewcommand{\arraystretch}{1.7}
\setlength\arraycolsep{3.5pt}
\mathbf{P} = \begin{bmatrix}
 0 & 0 & 0 &  \cdots & \\
 & &  \vdots & &  \text{(} n+2p\text{ rows of top padding)}\\
 0 & 0 & 0 & \cdots & \\
 & &  \vdots & &  \text{(} p\text{ rows of left padding)}\\
 1 & 0 & 0 & \cdots &  \\
 0 & 1 & 0 &   \cdots & \\
 & &  \vdots & &   \text{(} n\text{ rows of identity)}\\
 0 & 0 & 0 & \cdots & \\
 & &  \vdots & &  \text{(} p\text{ rows of right padding)}\\
 & &  \vdots & & \text{($m$ repeated blocks)}\\
 0 & 0 & 0 &  \cdots &\\
 & &  \vdots & &  \text{(} n+2p\text{ rows of bottom padding)}
\end{bmatrix}
\end{align*}

$\mathbf{P}$ is a block matrix characterized by leading rows of zeros to represent the top row of padding and then blocks of the identity matrix between rows of zeros, which capture each row of $\mathbf{A}$ with padding on either side, followed by trailing rows of zeros to represent the bottom row of padding. We next consider convolution on $\Vec{a}_{\text{pad}}$. We aim to build a matrix form of the convolution with the kernel, $\mathbf{C}$. First, let's consider the kernel:

\begin{equation}
\mathbf{K} = \begin{bmatrix}
K_{11} & \cdots & K_{1k} \\
\vdots & \ddots & \vdots \\
K_{k1} &  \cdots & K_{kk} \\ 
\end{bmatrix} \nonumber \\
\end{equation}

The kernel begins in the upper left corner of the input, and each row of the kernel acts on the first $k$ columns of the corresponding row of the input, from rows $1$ to $k$. Thus, the first row of $\mathbf{C}$ is given by:

\begin{align*}
\begin{bmatrix}
    K_{11} & \cdots & K_{1k} & 0 & \cdots & \text{(} n+2p-k\text{ zeros)} & \cdots
\end{bmatrix}
\end{align*}
\begin{center}
    \text{(repeated} $k$ \text{times with increasing rows of $\mathbf{K}$)}
\end{center}

After the repeats, there are $(n+2p)(m+2p-k)$ trailing zeros (to capture the rows of the input to which the kernel is not applied). The kernel then slides $s$ columns to the right over the input. This shifts the entries of the first row of $\mathbf{C}$ by $s$ zeros. The second row (displayed split onto two lines) is thus:

\begin{align*}
\setlength\arraycolsep{3.5pt}
\begin{bmatrix}
    0 & \cdots & \text{(}s \text{ zeros)} &K_{11} & \cdots & K_{1k}  \\
    & & 0 & \cdots & \text{(} n+2p-k-s\text{ zeros)} & \cdots
\end{bmatrix}
\end{align*}
\begin{center}
    \text{(repeated} $k$ \text{times with increasing rows of $\mathbf{K}$)}
\end{center}

There are the same trailing zeros as before. This sliding continues in the same form until the kernel reaches the end of the input matrix. The rightmost entry of $\mathbf{K}$ begins at $A_{\text{pad}_{1k}}$. There are $n+2p-k$ columns following this position along the row. Thus, the kernel can slide a maximum of $\lfloor (n+2p-k) / s\rfloor$ times. If there is a remainder in the division, then there are trailing zeros, representing the region the kernel was unable to slide over. There will be $r=n+2p-k - s \lfloor (n+2p-k) / s\rfloor$ trailing zeros.

The next step is moving the kernel vertically. This involves adding $(n+2p)s$ leading zeros to the start and removing $(n+2p)s$ trailing zeros from the ends of the previously determined row patterns (representing moving the kernel vertically by $s$ rows). We then repeat the previously shown pattern for horizontal sliding. Thus, the first row after the first vertical slide becomes (split onto two lines):

\begin{align*}
\setlength\arraycolsep{3pt}
\begin{bmatrix}
    0  \cdots  \Big((n+2p)s \text{ zeros} \Big)&  \\
    \Big\{K_{11}  \cdots  K_{1k} &0  \cdots \text{(} n+2p-k\text{ zeros)}\Big\}  \cdots
\end{bmatrix}
\end{align*}
\begin{center}
    \text{($\{\}$ repeated} $k$ \text{times with increasing rows of $\mathbf{K}$)}
\end{center}

\begin{center}
    \text{(with $(n+2p)(m+2p-k-s)$ trailing zeros)}
\end{center}

Here, $\{\}$ in the caption represents the expression in curly brackets within the row. We continue this process for additional vertical slides on the input; however, once we reach the end while vertically sliding the kernel, we may have the same remainder effect as with horizontal sliding, since we can only slide $\lfloor (m+2p-k) / s\rfloor$ times. In the absence of the remainder, the set of rows at the furthest vertical position representing the last horizontal slide of the kernel will have no trailing zeros (the last $\lfloor (n+2p-k) / s\rfloor+1 $ rows of $\mathbf{C}$). If there is a remainder, these rows will have $(n+2p)q$ trailing zeros, where $q=m+2p-k - s \lfloor (m+2p-k) / s\rfloor$.

We can capture the overall process as a block matrix. We have the repeating patterns captured in the matrix, $\mathbf{R}$, which represents the sliding of the kernel from left to right, as described in the demonstration for the first two rows of $\mathbf{C}$. $\mathbf{R}$ includes the $r$ trailing zeros on each row to account for the remainder when sliding horizontally. $\mathbf{R}$ is an $(\lfloor (n+2p-k) / s\rfloor + 1) \times (n+2p)k$ matrix. We can capture vertical sliding by representing the changing leading/trailing zeros as $(\lfloor (n+2p-k) +1 )/ s\rfloor \times (n+2p)s$ zero matrices, which we refer to as $\mathbf{0}_{\text{vert}}$. However, as noted previously, there may be a remainder when vertically sliding, so we define $\mathbf{0}_{q}$ as $(\lfloor (n+2p-k) / s\rfloor + 1) \times (n+2p)q$ to account for it. Furthermore, we define: $v=\lfloor (m+2p-k) / s\rfloor$. The overall matrix is constructed as:
\begin{align}
\mathbf{C=}
\begin{bmatrix}
    \mathbf{R} & \mathbf{0}_\text{vert} &\cdots &  (v \times) & &\mathbf{0}_q\\
    \mathbf{0}_\text{vert} & \mathbf{R} &  \mathbf{0}_\text{vert} & \cdots &\Big((v-1) \times\Big) & \mathbf{0}_q \\
    \vdots & & \ddots & &  &\vdots \\
    \mathbf{0}_\text{vert} &\cdots &(v \times)  &  \mathbf{R}  & &\mathbf{0}_q\\
\end{bmatrix}
\end{align}
Here, the expressions in parentheses express the number of repeated units of $\mathbf{0}_\text{vert}$. Finally, we have that for input $\mathbf{A}$, padding by $p$ and convolution with $\mathbf{K}$ with a stride of $s$ can be represented as:
\begin{align}
    \text{vec}^{-1}\Big(\mathbf{CP}\text{ vec}(\mathbf{A})\Big)
\end{align}

Both $\mathbf{C}$ and $\mathbf{P}$ are sparse matrices and very sparse for large inputs and small kernels, as are common in image processing/machine learning contexts. Thus, we can leverage a sparse matrix representation of the product $\mathbf{CP}$ and perform SpMV. Sparse representations store only the non-zero entries of the matrix. 

There is a one-time cost of creating $\mathbf{CP}$, after which the number of non-zero entries quantifies the total number of multiplications performed in the convolution operation and thus the time complexity of the algorithm. The total number of multiplications in current methods, including with zeros, is given by $m_\text{out}\times n_{\text{out}} \times k^2$, where $m_{\text{out}} =\lfloor (m+2p-k) / s\rfloor +1$ and $n_{\text{out}} =\lfloor (n+2p-k) / s\rfloor +1$ are the dimensions of the output. Each output element corresponds to a position of the kernel, so there are $k^2$ multiplications performed for each element of the output. 

Current implementations largely use the \texttt{im2col} operation, which converts the input into a matrix whose columns correspond to patches of the input upon which the kernel operates \cite{chellapilla:inria-00112631}. These methods have time complexity $O(m_\text{out}\times n_{\text{out}} \times k^2 )$. Our method, in contrast, excludes zero entries, which can constitute a significant number of the computations (particularly when padding constitutes a significant fraction of the matrix, as in the low-dimensional feature maps extracted by CNNs). We now derive the total number of non-zero multiplications. 

\begin{theorem} 
There are at most:
\begin{equation*}
\sum_{x=0}^{m_{\textnormal{out}}-1} \textnormal{ } \sum_{y=0}^{n_{\textnormal{out}}-1} \textnormal{max}\Big(0, k-c_{1}(x)\Big) 
\textnormal{ max}\Big(0, k-c_{2}(y)\Big)
\end{equation*}
non-zero multiplications when convolving a $k \times k$ kernel with an $m \times n$ input with padding $p$ and stride $s$, where:
\begin{align*}
c_{1}(x) &= \textnormal{max}(0, p-sx) + \textnormal{max}(0, sx+k-m-p) \\
c_{2}(y) &= \textnormal{max}(0, p-sy) + \textnormal{max}(0, sy+k-n-p) \\
m_{\textnormal{out}} &=\lfloor (m+2p-k) / s\rfloor +1 \\
n_{\textnormal{out}} &=\lfloor (n+2p-k) / s\rfloor +1 \\
k &\leq m+2p, \text{ }n+2p
\end{align*}\end{theorem}
\begin{proof}
As a preliminary, we state ``at most'' because the kernel or the input may contain additional zeros, beyond those in the padding. Now consider the starting position of the kernel in the upper left corner of the input. In the first case that $p<k$, there are $k-p$ non-zero entries within the receptive field of the kernel horizontally (columns) and vertically (rows) and thus $(k-p)^2$ non-zero multiplications. 

Sliding to the next horizontal position, there will still be the same top padding of $p$ zeros, but if $p>s$, then there will be $p-s$ zeros remaining of left padding. However, if $p<s$, then there will be no left padding remaining. Thus, there will be $k-p$ rows of non-zero entries and $k-\text{max}(0, p-s)$ columns of non-zero entries, giving $(k-p)(k-\text{max}(0, p-s))$ non-zero multiplications. We continue sliding horizontally in this manner, incrementing the zeros in the left padding by the stride, such that we have $(k-p)(k-\text{max}(0, p-sy))$ non-zero multiplications, at each horizontal slide of the kernel, $y$. 

However, if any region of the kernel occupies the columns beyond $n+p$, then there will be columns of zeros of right padding. We can account for these columns by counting the columns of the kernel at each horizontal slide, $y$, that extend beyond this region with $sy+k-(n+p)$. The additional $k$ term accounts for the width of the kernel beyond the position $sy$. Therefore, if $sy+k>n+p$, then there will be $k-(sy+k-n-p)$ columns of non-zero elements and the same $k-p$ rows of non-zero elements. Meanwhile, if $sy+k<n+p$, then there is no right padding. Thus, the total number of non-zero multiplications accounting for right padding is given by: $(k-p)(k-\text{max}(0, sy+k-n-p))$. We know that the total number of horizontal slides is given by $n_{\text{out}}-1 = \lfloor (n+2p-k) / s\rfloor$. If we combine our expressions for left and right padding over the first horizontal slide, we have:
\begin{equation*}
\sum_{y=0}^{n_{\textnormal{out}}-1} ( k-p) 
(k-c_{2}(y))
\end{equation*}
with: 
\begin{equation*}
c_{2}(y) = \textnormal{max}(0, p-sy) + \textnormal{max}(0, sy+k-n-p)
\end{equation*}

Now we consider vertical sliding, again for the $p<k$ case. As we slide vertically, the top padding shrinks. Similar to the horizontal case, if $p>s$, then there will be $p-s$ remaining rows of top padding, while if $p<s$, there will be no top padding of zeros remaining. So we have $k-\text{max}(0, p-s)$ rows of top padding after the first vertical slide. We can increment vertical slides with $x$, such that for each vertical slide, we will have $k-\text{max}(0, p-sx)$ rows of top padding. Now we consider bottom padding similar to the horizontal case. If any region of the kernel occupies the rows beyond $m+p$, then there will be rows of zeros of bottom padding. We can account for these rows with the same reasoning as above, giving $\text{max}(0, sx+k-m-p)$ rows of bottom padding at any vertical slide, $x$. Therefore, the number of non-zero rows at any vertical position is given by:
\begin{equation*}
k-c_{1}(x)
\end{equation*}
with: 
\begin{align*}
c_{1}(x) = \textnormal{max}(0, p-sx) + \textnormal{max}(0, sx+k-m-p)
\end{align*}

Combining our expressions for vertical and horizontal slides and summing at each position $(x,y)$, for $p<k$, the number of non-zero multiplications is given by:
\begin{equation}
\sum_{x=0}^{m_{\textnormal{out}}-1} \textnormal{ } \sum_{y=0}^{n_{\textnormal{out}}-1} \Big(k-c_{1}(x)\Big) \Big(k-c_{2}(y)\Big)
\end{equation}

Now we extend briefly to the $p>k$ case. If $p>k$, then there will be regions where the kernel only overlaps with zero entries, and there will be no non-zero multiplications in these regions. These regions will occur when the kernel is in a position within the padding, or when $p-sx>k$, $sx+k-m-p>k$, $p-sy>k$, or $sy+k-n-p>k$, which would result in negative values in our expressions for the number of non-zero rows and columns in Equation 2.8. In these cases, there are no non-zero multiplications; thus, we modify Equation 2.8 with our final result for the total number of non-zero multiplications:
\begin{equation}
\sum_{x=0}^{m_{\textnormal{out}}-1} \textnormal{ } \sum_{y=0}^{n_{\textnormal{out}}-1} \textnormal{max}\Big(0, k-c_{1}(x)\Big) 
\textnormal{ max}\Big(0, k-c_{2}(y)\Big)
\end{equation}
\end{proof}

\section{Algorithm Implementation}
First, we define an algorithm for constructing the padding matrix, $\mathbf{P}$, in Algorithm 1. We construct a sparse matrix (SpM) representation of $\mathbf{P}$ by defining the indices of the ones (which are the only non-zero entries) within $\mathbf{P}$. Next, we define an algorithm for constructing the convolution matrix, $\mathbf{C}$, in Algorithm 2. We present a method that constructs a dense matrix of the repeating pattern, $\mathbf{R}$, and then converts it to a sparse matrix, though this could be optimized as in Algorithm 1 to simply calculate the indices of the non-zero values. However, since we only have a one-time construction cost of $\mathbf{C}$ and this work is focused on the time complexity of the convolution operation itself, we do not derive such an algorithm. Finally, we restate Equation 2.7 in Algorithm 3. At this stage, we leave the exact SpM representation ambiguous. In the following section, we experimentally consider two different SpM representations that are optimized for SpMV: Compressed Sparse Row (CSR) and Compressed Sparse Column (CSC).

\begin{algorithm}[t]
\caption{Padding Matrix Construction}
\begin{algorithmic}[1]
\Procedure{$\mathbf{P}$}{$m, n, p$}
    \State $r_{\text{out}} \gets (m + 2 \times p) \times (n + 2 \times p)$
    \State $c_{\text{out}} \gets m \times n$
    \State $R_{\text{ind}} \gets \text{zeros}(c_{\text{out}})$ \Comment{array of zeros of size $c_{\text{out}}$}
    \State $C_{\text{ind}} \gets \text{zeros}(c_{\text{out}})$ \Comment{array of zeros of size $c_{\text{out}}$}
    \State $V_\text{val} \gets \text{ones}(c_{\text{out}})$ \Comment{array of ones of size $c_{\text{out}}$}
    \vspace{0.5cm}
    \For{$i = 1$ \textbf{to} $m$}
        \State $r_{\text{start}} \gets (n + 2 \times p) \times (i + p - 1) + p$
        \State $r_{\text{end}} \gets r_{\text{start}} + n$
        \State $c_{\text{start}} \gets n \times (i - 1)$
        \State $c_{\text{end}} \gets n \times i$
        \vspace{0.5cm}
        \newline \Comment{Square brackets $[a:b]$ indicate indexing an array from $a$ (inclusive) to $b$ (exclusive).}
        \newline \Comment{$\text{arange}(a,b)$ creates an array from $a$ (inclusive) to $b$ (exclusive) with a step of 1.}
        \vspace{0.5cm}
        \State $R_{\text{ind}}[n \times (i-1):n \times i] \gets \text{arange}(r_{\text{start}}, r_{\text{end}})$
        \State $C_{\text{ind}}[n \times (i-1):n \times i] \gets \text{arange}(c_{\text{start}}, c_{\text{end}})$
    \EndFor
    \vspace{0.5cm}
    \newline \Comment{$\text{SpM}\Big((a,(b,c), (d,e)\Big)$ creates a sparse matrix with values in array $a$ whose indices are given in row and column arrays $(b,c)$ and whose overall size is $d \times e$.}
    \vspace{0.5cm}
    \State \Return \text{SpM}\Big(($V_{\text{val}}, (R_{\text{ind}}, C_{\text{ind}})), (r_{\text{out}}, c_{\text{out}})$\Big)
\EndProcedure
\end{algorithmic}
\end{algorithm}

\begin{algorithm}[htbp]
\caption{Convolution Matrix Construction}
\begin{algorithmic}[1]
\Procedure{$\mathbf{C}$}{$\mathbf{K}, k, m, n, p,s$}
    \State $w \gets  n+2p-k$
    \State $h \gets  m+2p-k$
    \newline \Comment{$R_{\text{rows}}$ is the number of rows in $\mathbf{R}$.}
    \State $R_{\text{rows}} \gets \lfloor w/s\rfloor +1$
    \newline \Comment{$B_{\text{columns}}$ is the number of columns of blocks in $\mathbf{C}$, excluding the remainder.}
    \State $B_{\text{columns}} \gets \lfloor h/s\rfloor+1$
    %\State $h_{\text{rem}} \gets w - s(R_{\text{rows}}-1)$
    \newline \Comment{$v_{\text{rem}}$ is the number of remainder rows when vertically sliding.}
    \State $v_{\text{rem}} \gets h - s(B_{\text{columns}}-1)$
    \vspace{0.17cm}
    \State $\mathbf{R} \gets [\text{ }\text{ }\text{ }]$ \Comment{Initialize $\mathbf{R}$ as empty.}
    \newline \Comment{Create $\mathbf{R}$ by looping over horizontal slides.}
    \For{$i = 0$ \textbf{to} $R_{\text{rows}}-1$} 
        \newline \Comment{Create $r_{\text{row}}$ by looping over kernel rows.}
        \State $r_{\text{row}} \gets [\text{ }\text{ }\text{ }] $ \Comment{Initialize $r_{\text{row}}$ as empty.}
        \For{$j = 1$ \textbf{to} $k$}
            \State $\mathbf{0}_{\text{start}} = [0\cdots(\times si)]$
            \State $\mathbf{0}_{\text{end}} = [0\cdots(\times (w-si))]$
            \State $\mathbf{K}_{\text{row}} = [K_{j1} \cdots  K_{jk}]$
            \State $t_{\text{row}} \gets \begin{bmatrix} \mathbf{0}_{\text{start}} & \mathbf{K}_{\text{row}} &\mathbf{0}_{\text{end}} \end{bmatrix}$ 
            \newline \Comment{+= indicates array concatenation.}
            \State $r_{\text{row}} += t_\text{row} $ 
        \EndFor
        \newline \Comment{Set the $i$-th row of $\mathbf{R}$ as $r_{\text{row}}$.}
        \State $\mathbf{R}[i] \gets r_{\text{row}}$
    \EndFor
    \vspace{0.17cm}
    \newline \Comment{$\text{SpM}(\mathbf{R})$ creates a sparse matrix that records the non-zero values and their positions from within $\mathbf{R}$.}
    \vspace{0.17cm}
    \State $\mathbf{R} \gets \text{SpM}(\mathbf{R})$
    \vspace{0.17cm}
    \newline \Comment{$\text{SpM}(a,b)$ creates an empty sparse matrix with size $(a,b)$.}
    \State $\mathbf{0}_{\text{vert}} = \text{SpM}(R_{\text{rows}}, (n+2p)s)$
    \State $\mathbf{0}_{\text{rem}} = \text{SpM}(R_{\text{rows}}, (n+2p)v_{\text{rem}})$
    \vspace{0.17cm}
    \newline \Comment{Create $\mathbf{C}_{\text{block}}$, an array of sparse matrices, by looping over vertical slides.}
    \State $\mathbf{C}_{\text{block}} \gets [\text{ }\text{ }\text{ }]$ \Comment{Initialize $\mathbf{C}_{\text{block}}$ as empty.}
    \For{$i = 0$ \textbf{to} $B_{\text{columns}}-1$}
        \State $\mathbf{0}_{\text{start}} \gets [\mathbf{0}_{\text{vert}} \cdots(\times i)]$
        \State $\mathbf{0}_{\text{end}} \gets [\mathbf{0}_{\text{vert}} \cdots(\times (B_{\text{columns}}-1-i)) \text{ } \text{ } \mathbf{0}_{\text{rem}}]$
        \State $\mathbf{C}_{\text{block, row}} \gets \begin{bmatrix} \mathbf{0}_{\text{start}} & \mathbf{R} &\mathbf{0}_{\text{end}} \end{bmatrix}$
        \newline \Comment{Set the $i$-th row of $\mathbf{C}_{\text{block}}$ as $\mathbf{C}_{\text{block, row}}$.}
        \State $\mathbf{C}_{\text{block}}[i] \gets \mathbf{C}_{\text{block, row}}$
    \EndFor
    \vspace{0.17cm}
    \newline \Comment{$\text{blocks}(\mathbf{C}_{\text{block}})$ creates a sparse block matrix from an array of sparse matrices.}
    \vspace{0.17cm}
    \State $\mathbf{C} = \text{blocks} (\mathbf{C}_{\text{block}})$
    \State \Return$\mathbf{C}$
\EndProcedure
\end{algorithmic}
\end{algorithm}

\begin{algorithm}[b]
\caption{Convolution Calculation}
\begin{algorithmic}[b]
    \State{Initialize and Compute: $\mathbf{P}, \mathbf{C}, \text{and } \mathbf{T}=\mathbf{CP}$ (the full transformation as a sparse matrix)}
    \Procedure{Convolution}{$\mathbf{A}, \mathbf{T}$}
    \State{$\Vec{a} = \text{vec}(\mathbf{A})$} \Comment{Also called $\text{ravel()}$ or $\text{view()}$}
    \State{$\Vec{a}_{\text{conv}} = \mathbf{T}\Vec{a}$}
    \State{$\mathbf{A}_{\text{conv}} = \text{vec}^{-1}(\Vec{a}_{\text{conv}})$} \Comment{Also called $\text{reshape()}$}
    \State \Return $\mathbf{A}_{\text{conv}}$
    \EndProcedure
\end{algorithmic}
\end{algorithm}

\section{Experimental Results} As a proof-of-concept, we compare five implementations of our algorithm (Algorithm 3) to \texttt{Conv2D}, an \texttt{im2col} convolution implementation in the popular deep learning framework, \texttt{PyTorch} \cite{10.5555/3454287.3455008}. We consider the cascade of convolutions in a popular deep CNN, DenseNet121, as a test case \cite{8099726}. 

\subsection{Materials and Methods}
We use the \texttt{SciPy} and \texttt{NumPy} libraries to create CSC and CSR versions of the algorithm for CPU, which we refer to as CSC-C and CSR-C respectively \cite{2020SciPy-NMeth, harris2020array}. We use \texttt{CuPy} to create CSC and CSR versions of the algorithm for GPU, which we refer to as CSC-G and CSR-G respectively, and we use \texttt{PyTorch} to create an additional CSR version, which we refer to as CSR-T \cite{cupy_learningsys2017, 10.5555/3454287.3455008}. We refer to the CPU and GPU versions of \texttt{Conv2D} as \texttt{Conv2D}-C and \texttt{Conv2D}-G respectively.

We determine the single-channel input matrix size ($m \times n$), kernel size ($k \times k$), stride ($s$), and padding ($p$) for each layer of the feature extraction portion of DenseNet121, given a standard $224 \times 224$, three-channel input to the model. Then we model the convolution (or pooling, which we also represent as a convolution operation) at each layer by generating a random normal $m \times n$ input matrix and random normal $k \times k$ kernel. Max-pooling cannot be accurately represented as a kernel operation; however, for the purposes of this demonstration, we model it as such. We perform all five implementations of our algorithm and \texttt{Conv2D}-C and \texttt{Conv2D}-G with these parameters for 10,000 trials.

Code for experiments is available at: \url{https://github.com/ZanChaudhry/SpMV_Conv}. Experiments were conducted on a machine equipped with the following hardware:
\begin{itemize}
    \item \textbf{CPU:} Intel Core i7-8700 (6 cores, 12 threads, 3.20 GHz base clock)
    \item \textbf{GPU:} NVIDIA GeForce GTX 1070 (8 GB GDDR5, CUDA Compute Capability 6.1)
\end{itemize}
\subsection{Results and Discussion}
We present the results for CPU-based implementations in Figure 1 and for GPU-based implementations in Figure 2. Both figures are included as full-pages at the end of the paper. We also sum all layer times for each trial of each implementation, and we calculate the mean total time. We present these times in Table 1. In Table 2, we present the layer details, $(m, n, k, s, p)$, for each layer. Table 2 is included as a full-page figure at the end of the paper.

On CPU (Fig. 1), one of the implementations of our algorithm (CSC-C or CSR-C) gives the best performance at each layer. However, the optimal implementation varies from layer to layer. The total time (Table 1) is very similar for both implementations and roughly half of the time of \texttt{Conv2D}-C. 

On GPU (Fig. 2), all GPU implementations present significant speed improvements on the first layer, likely due to increased parallelizability at large matrix sizes (roughly $3.6\times$ for implementations of our algorithm and $5.5\times$ for \texttt{Conv2D}-G). At this layer, \texttt{Conv2D}-G gives the best performance by a considerable margin. At all other layers, CSR-T gives the best performance, though with much smaller margins than on the CPU. CSR-T is generally followed by \texttt{Conv2D}-G, and lastly with CSC-G, CSR-G giving similar performances. The total times in Table 1 also reflect this fact. Likely, the \texttt{PyTorch} SpMV implementation is more optimized on GPU than \texttt{CuPy}.

The considerable performance of \texttt{Conv2D}-G at the first layer likely arises from the increased parallelizability of General Matrix Multiplications (GEMMs) in comparison to SpMVs, which require additional logical operations for indexing \cite{gao2024systematicliteraturesurveysparse}. Through sparse formats that leverage the block structure of $\mathbf{C}$, more efficient parallelization may be possible to make our method considerably more competitive at larger sizes on the GPU in future work. Additional GPU-level improvements may be possible, as sparse formats are a recent addition in \texttt{PyTorch}, whereas \texttt{Conv2D} has been considerably optimized \cite{pytorch_sparse}. Aside from this first layer, the GPU-based implementations are significantly slower than those on CPU, as at smaller sizes, the many parallel cores of the GPU cannot be used efficiently compared to the much more powerful CPU cores \cite{https://doi.org/10.1049/joe.2018.9178}. This paper serves as a proof-of-concept mainly to present the mathematical motivation behind the algorithm, so we investigate a minimal test problem; however, in future work we can explore the multi-channel and batched contexts in which GPU implementations provide considerable advantages. 

Overall, we see significant speed improvements from applying our algorithms (particularly on the CPU). Likely, this stems from the required matrix creation step in the \texttt{im2col} method, whereas our pre-computed transformation matrix simply requires some view/reshape operations and an SpMV operation. Our method is disadvantageous, however, in contexts where the kernel is changing (for example CNN training), and thus the transformation must be repeatedly computed. But in fixed kernel settings (for example CNN inference), our method can significantly improve computation speed, potentially with great performance improvements in time-sensitive settings (for example real-time CNN inference). 

\begin{table}[htbp]
    \centering
    \caption{Mean Total Execution Time Over All Layers of Simulated DenseNet121 ($\mu \pm \sigma/\sqrt{n}$, $n=10,000$)}
    \vspace{0.2cm}
    \begin{tabular}{@{}c|c@{}} % Adjusted for swapped rows and columns
        \toprule
         Algorithm & Time (\SI{}{\micro\second}) \\ \midrule
         CSR-T & $\mathbf{9,187 \pm 5}$\\
         CSC-G &$11,740 \pm 7$ \\
         CSR-G &$11,737 \pm 4$ \\
         \texttt{Conv2D}-G &$11,143 \pm 3$ \\
         \midrule
         CSC-C &$1,536 \pm 1$ \\
         CSR-C &$\mathbf{1,507 \pm 1}$ \\
         \texttt{Conv2D}-C &$2,871 \pm 4$ \\
         \bottomrule
    \end{tabular}
\end{table}

\section{Conclusion} Here we present a simple algorithm for representing convolution with zero-padding and stride as a sparse transformation matrix applied a vectorized input in a sparse matrix vector multiplication (SpMV). Furthermore, in Theorem 2.1, we contribute an explicit expression for the number of non-zero multiplications in a convolution with stride and padding, which is relevant for future work on leveraging sparsity in convolution algorithms. We provide a proof-of-concept implementation of our algorithm and assess its performance on CPU and GPU. Future work can explore more efficient sparse matrix representations and multiplication methods for increased parallelizability in GPU settings. Overall, we add knowledge and methods to the ongoing exploration of algorithmic methods for exploiting sparsity in convolutions with padding.

%\begin{thebibliography}{99}

\bibliographystyle{siam} 
\bibliography{references}

\clearpage

\begin{table*}[p]
    \footnotesize
    \centering
    \caption{DenseNet121 Layer Information (Read Vertically, then Across)}
    \vspace{0.2cm}
    \begin{tabular}{@{}c|c|c|c|c|c@{}} % Adjusted for swapped rows and columns
        \toprule
         Layer Name & $(m,n,k,s,p)$ & Layer Name (Cont.) &  $(m,n,k,s,p)$ & Layer Name (Cont.) &  $(m,n,k,s,p)$ \\
         \midrule
         conv0 & $(224, 224, 7, 2, 3)$ & block3.layer1.conv2 & $(14, 14, 3, 1, 1)$ & block3.layer22.conv2 & $(14, 14, 3, 1, 1)$ \\
         pool0 & $(112, 112, 3, 2, 1)$ &  block3.layer2.conv1 & $(14, 14, 1, 1, 0)$ &  block3.layer23.conv1 & $(14, 14, 1, 1, 0)$ \\
         block1.layer1.conv1 & $(56, 56, 1, 1, 0)$ & block3.layer2.conv2 & $(14, 14, 3, 1, 1)$ & block3.layer23.conv2 & $(14, 14, 3, 1, 1)$ \\
         block1.layer1.conv2 & $(56, 56, 3, 1, 1)$ & block3.layer3.conv1 & $(14, 14, 1, 1, 0)$ & block3.layer24.conv1 & $(14, 14, 1, 1, 0)$ \\
         block1.layer2.conv1 & $(56, 56, 1, 1, 0)$ & block3.layer3.conv2 & $(14, 14, 3, 1, 1)$ & block3.layer24.conv2 & $(14, 14, 3, 1, 1)$ \\
         block1.layer2.conv2 & $(56, 56, 3, 1, 1)$ & block3.layer4.conv1 & $(14, 14, 1, 1, 0)$ & transition3.conv & $(14, 14, 1, 1, 0)$ \\
         block1.layer3.conv1 & $(56, 56, 1, 1, 0)$ & block3.layer4.conv2 & $(14, 14, 3, 1, 1)$ & transition3.pool & $(14, 14, 2, 2, 0)$ \\
         block1.layer3.conv2 & $(56, 56, 3, 1, 1)$ & block3.layer5.conv1 & $(14, 14, 1, 1, 0)$ & block4.layer1.conv1 & $(7, 7, 1, 1, 0)$ \\
         block1.layer4.conv1 & $(56, 56, 1, 1, 0)$ & block3.layer5.conv2 & $(14, 14, 3, 1, 1)$ & block4.layer1.conv2 & $(7, 7, 3, 1, 1)$ \\
         block1.layer4.conv2 & $(56, 56, 3, 1, 1)$ & block3.layer6.conv1 & $(14, 14, 1, 1, 0)$ & block4.layer2.conv1 & $(7, 7, 1, 1, 0)$ \\
         block1.layer5.conv1 & $(56, 56, 1, 1, 0)$ & block3.layer6.conv2 & $(14, 14, 3, 1, 1)$ & block4.layer2.conv2 & $(7, 7, 3, 1, 1)$ \\
         block1.layer5.conv2 & $(56, 56, 3, 1, 1)$ & block3.layer7.conv1 & $(14, 14, 1, 1, 0)$ & block4.layer3.conv1 & $(7, 7, 1, 1, 0)$ \\
         block1.layer6.conv1 & $(56, 56, 1, 1, 0)$ & block3.layer7.conv2 & $(14, 14, 3, 1, 1)$ & block4.layer3.conv2 & $(7, 7, 3, 1, 1)$ \\
         block1.layer6.conv2 & $(56, 56, 3, 1, 1)$ & block3.layer8.conv1 & $(14, 14, 1, 1, 0)$ & block4.layer4.conv1 & $(7, 7, 1, 1, 0)$ \\
         transition1.conv & $(56, 56, 1, 1, 0)$ & block3.layer8.conv2 & $(14, 14, 3, 1, 1)$ & block4.layer4.conv2 & $(7, 7, 3, 1, 1)$ \\
         transition1.pool & $(56, 56, 2, 2, 0)$ & block3.layer9.conv1 & $(14, 14, 1, 1, 0)$ & block4.layer5.conv1 & $(7, 7, 1, 1, 0)$ \\
         block2.layer1.conv2 & $(28, 28, 3, 1, 1)$ & block3.layer9.conv2 & $(14, 14, 3, 1, 1)$ & block4.layer5.conv2 & $(7, 7, 3, 1, 1)$ \\
         block2.layer2.conv1 & $(28, 28, 1, 1, 0)$ & block3.layer10.conv1 & $(14, 14, 1, 1, 0)$ & block4.layer6.conv1 & $(7, 7, 1, 1, 0)$ \\
         block2.layer2.conv2 & $(28, 28, 3, 1, 1)$ & block3.layer10.conv2 & $(14, 14, 3, 1, 1)$ & block4.layer6.conv2 & $(7, 7, 3, 1, 1)$ \\
         block2.layer3.conv1 & $(28, 28, 1, 1, 0)$ & block3.layer11.conv1 & $(14, 14, 1, 1, 0)$ & block4.layer7.conv1 & $(7, 7, 1, 1, 0)$ \\
         block2.layer3.conv2 & $(28, 28, 3, 1, 1)$ & block3.layer11.conv2 & $(14, 14, 3, 1, 1)$ & block4.layer7.conv2 & $(7, 7, 3, 1, 1)$ \\
         block2.layer4.conv1 & $(28, 28, 1, 1, 0)$ & block3.layer12.conv1 & $(14, 14, 1, 1, 0)$ & block4.layer8.conv1 & $(7, 7, 1, 1, 0)$ \\
         block2.layer4.conv2 & $(28, 28, 3, 1, 1)$ & block3.layer12.conv2 & $(14, 14, 3, 1, 1)$ & block4.layer8.conv2 & $(7, 7, 3, 1, 1)$\\
         block2.layer5.conv1 & $(28, 28, 1, 1, 0)$ & block3.layer13.conv1 & $(14, 14, 1, 1, 0)$ & block4.layer9.conv1 & $(7, 7, 1, 1, 0)$ \\
         block2.layer5.conv2 & $(28, 28, 3, 1, 1)$ & block3.layer13.conv2 & $(14, 14, 3, 1, 1)$ & block4.layer9.conv2 & $(7, 7, 3, 1, 1)$ \\
         block2.layer6.conv1 & $(28, 28, 1, 1, 0)$ & block3.layer14.conv1 & $(14, 14, 1, 1, 0)$ & block4.layer10.conv1 & $(7, 7, 1, 1, 0)$ \\
         block2.layer6.conv2 & $(28, 28, 3, 1, 1)$ & block3.layer14.conv2 & $(14, 14, 3, 1, 1)$ & block4.layer10.conv2 & $(7, 7, 3, 1, 1)$ \\
         block2.layer7.conv1 & $(28, 28, 1, 1, 0)$ & block3.layer15.conv1 & $(14, 14, 1, 1, 0)$ & block4.layer11.conv1 & $(7, 7, 1, 1, 0)$ \\
         block2.layer7.conv2 & $(28, 28, 3, 1, 1)$ & block3.layer15.conv2 & $(14, 14, 3, 1, 1)$ & block4.layer11.conv2 & $(7, 7, 3, 1, 1)$ \\
         block2.layer8.conv1 & $(28, 28, 1, 1, 0)$ & block3.layer16.conv1 & $(14, 14, 1, 1, 0)$ & block4.layer12.conv1 & $(7, 7, 1, 1, 0)$ \\
         block2.layer8.conv2 & $(28, 28, 3, 1, 1)$ & block3.layer16.conv2 & $(14, 14, 3, 1, 1)$ & block4.layer12.conv2 & $(7, 7, 3, 1, 1)$ \\
         block2.layer9.conv1 & $(28, 28, 1, 1, 0)$ & block3.layer17.conv1 & $(14, 14, 1, 1, 0)$ & block4.layer13.conv1 & $(7, 7, 1, 1, 0)$ \\
         block2.layer9.conv2 & $(28, 28, 3, 1, 1)$ & block3.layer17.conv2 & $(14, 14, 3, 1, 1)$ & block4.layer13.conv2 & $(7, 7, 3, 1, 1)$ \\
         block2.layer10.conv1 & $(28, 28, 1, 1, 0)$ & block3.layer18.conv1 & $(14, 14, 1, 1, 0)$ & block4.layer14.conv1 & $(7, 7, 1, 1, 0)$ \\
         block2.layer10.conv2 & $(28, 28, 3, 1, 1)$ & block3.layer18.conv2 & $(14, 14, 3, 1, 1)$ & block4.layer14.conv2 & $(7, 7, 3, 1, 1)$ \\
         block2.layer11.conv1 & $(28, 28, 1, 1, 0)$ & block3.layer19.conv1 & $(14, 14, 1, 1, 0)$ & block4.layer15.conv1 & $(7, 7, 1, 1, 0)$ \\
         block2.layer11.conv2 & $(28, 28, 3, 1, 1)$ & block3.layer19.conv2 & $(14, 14, 3, 1, 1)$ & block4.layer15.conv2 & $(7, 7, 3, 1, 1)$ \\
         block2.layer12.conv1 & $(28, 28, 1, 1, 0)$ & block3.layer20.conv1 & $(14, 14, 1, 1, 0)$ & block4.layer16.conv1 & $(7, 7, 1, 1, 0)$ \\
         block2.layer12.conv2 & $(28, 28, 3, 1, 1)$ & block3.layer20.conv2 & $(14, 14, 3, 1, 1)$ & block4.layer16.conv2 & $(7, 7, 3, 1, 1)$ \\
         transition2.conv & $(28, 28, 1, 1, 0)$ & block3.layer21.conv1 & $(14, 14, 1, 1, 0)$  \\
         transition2.pool & $(28, 28, 2, 2, 0)$ & block3.layer21.conv2 & $(14, 14, 3, 1, 1)$   \\
         block3.layer1.conv1 & $(14, 14, 1, 1, 0)$ & block3.layer22.conv1 & $(14, 14, 1, 1, 0)$ \\

         \bottomrule
    \end{tabular}
\end{table*}

\clearpage

\begin{figure*}[p]
\centering
\includegraphics[height=0.94\textheight]{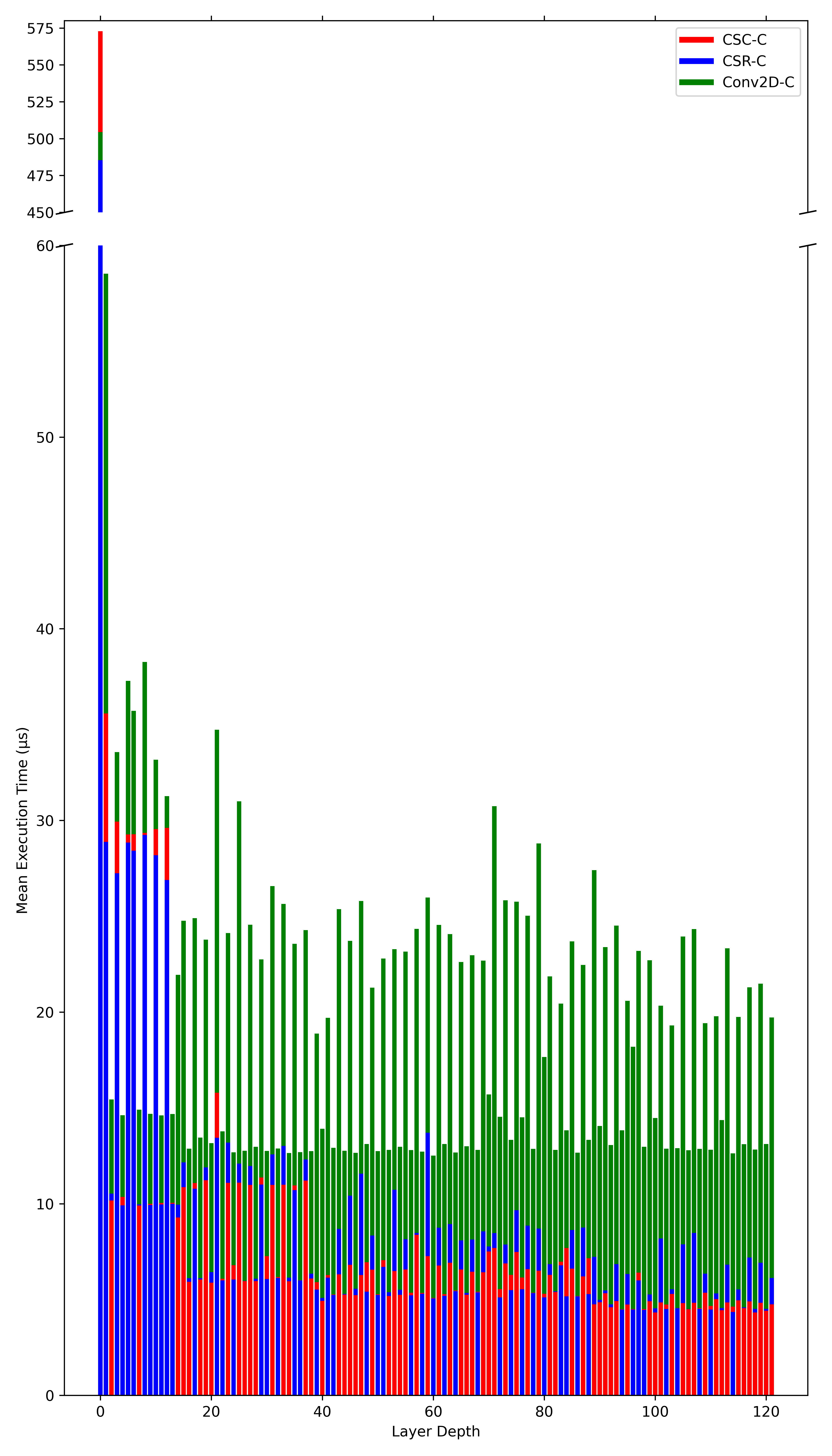}
\caption{Mean execution times $(n=10,000)$ for CPU experiments are presented over the simulated, randomly generated DenseNet121 layers. Columns are stacked to show time differences between implementations.}
\end{figure*}

\clearpage

\begin{figure*}[p]
\centering
\includegraphics[height=0.94\textheight]{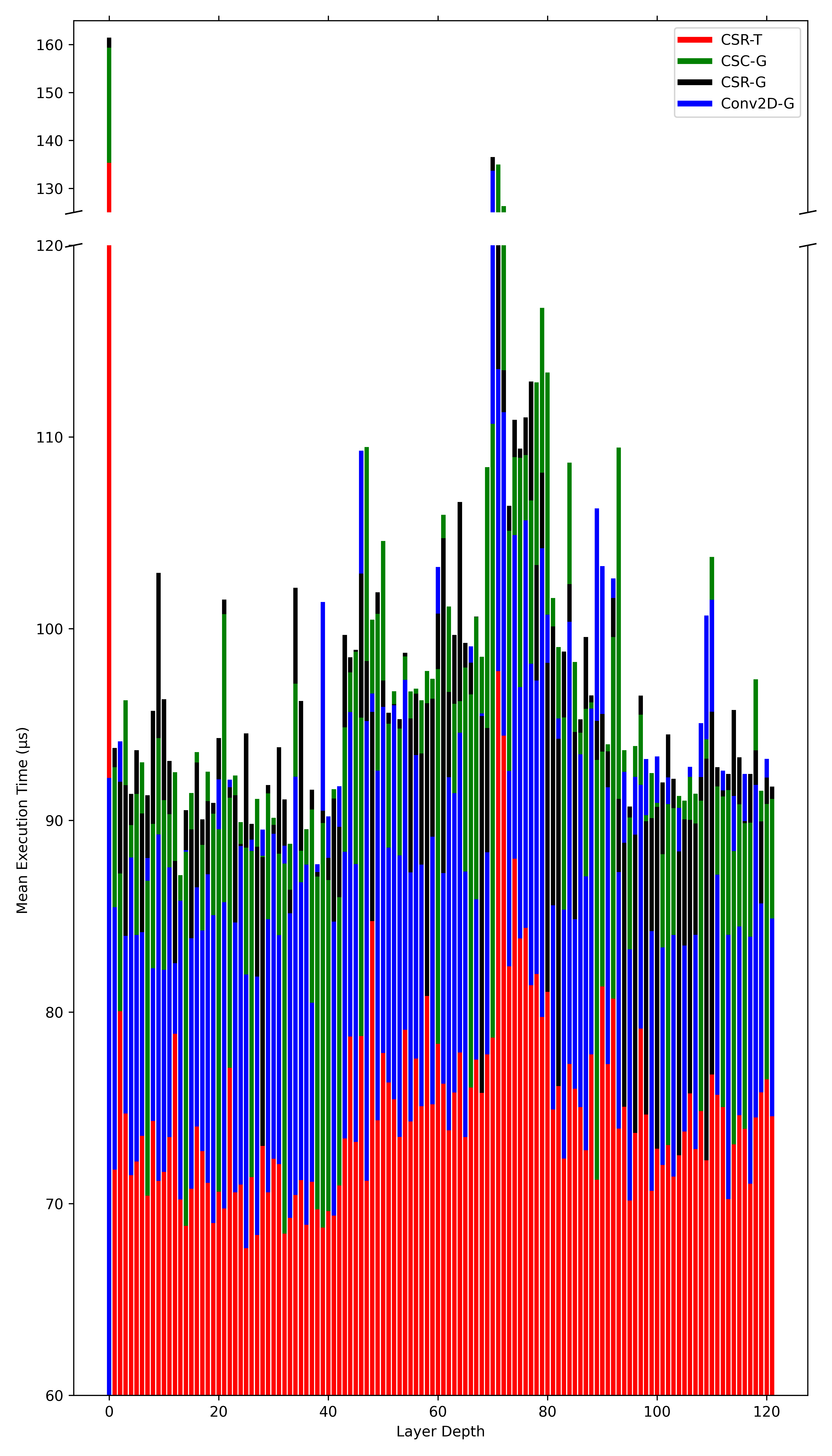}
\caption{Mean execution times $(n=10,000)$ for GPU experiments are presented over the simulated, randomly generated DenseNet121 layers. Columns are stacked to show time differences between implementations.}
\end{figure*}

\end{document}